\documentclass{article}

\usepackage[a4paper, total={6in, 9in}]{geometry}

\usepackage[colorlinks=true, linkcolor=black, citecolor=black, urlcolor=black]{hyperref}

\usepackage{bm}
\usepackage{url} 
\usepackage{dsfont}
\usepackage{xcolor}
\usepackage{amsthm}
\usepackage{amsmath}
\usepackage{amssymb}
\usepackage{multicol}
\usepackage{mathtools}
\usepackage{nicefrac}

\usepackage{algorithm}
\usepackage{algorithmic}

\newlength\myindent
\newcommand\bindent{%
  \begingroup
  \setlength{\itemindent}{\myindent}
  \addtolength{\algorithmicindent}{\myindent}
}
\newcommand\eindent{\endgroup}

\usepackage{tikz}
\usepackage{graphicx}
\usepackage{pgfplots}
\usepackage{adjustbox}
\usepackage{booktabs}

\pgfplotsset{compat=1.18}

\usepackage[round]{natbib}

\usepackage[toc, page, titletoc]{appendix}

\newtheorem{lemma}{Lemma}

\newtheorem{theorem}{Theorem}

\newtheorem{assumption}{Assumption}
\newtheorem{corollary}{Corollary}

\usepackage{thm-restate}

\DeclareMathOperator*{\argmin}{arg\,min}
\DeclarePairedDelimiter{\abs}{\lvert}{\rvert}

\title{QuACK: A Multipurpose Queuing Algorithm for Cooperative $k$-Armed Bandits}

\author{Benjamin Howson \and Sarah Filippi \and Ciara Pike-Burke}
\date{
    Department of Mathematics, Imperial College London\\[2ex]
    \today
}

\allowdisplaybreaks

\begin{document}
\maketitle

\begin{abstract}
We study the cooperative stochastic $k$-armed bandit problem, where a network of $m$ agents collaborate to find the optimal action. In contrast to most prior work on this problem, which focuses on extending a specific algorithm to the multi-agent setting, we provide a black-box reduction that allows us to extend any single-agent bandit algorithm to the multi-agent setting. Under mild assumptions on the bandit environment, we prove that our reduction transfers the regret guarantees of the single-agent algorithm to the multi-agent setting. These guarantees are tight in subgaussian environments, in that using a near minimax optimal single-player algorithm is near minimax optimal in the multi-player setting up to an additive graph-dependent quantity. Our reduction and theoretical results are also general, and apply to many different bandit settings. By plugging in appropriate single-player algorithms, we can easily develop provably efficient algorithms for many multi-player settings such as heavy-tailed bandits, duelling bandits and bandits with local differential privacy, among others. Experimentally, our approach is competitive with or outperforms specialised multi-agent algorithms.
\end{abstract}

\section{Introduction}
Stochastic multi-armed bandit problems are a fundamental model for sequential decision-making. Here, a single agent sequentially interacts with the environment by selecting an action and receiving a reward from an unknown distribution associated with that action \citep{Lattimore2020}. There are numerous provably efficient algorithms for this setting that employ different mechanisms for balancing the exploration-exploitation trade-off, such as optimism, posterior sampling, bootstrapping, and soft-max exploration \citep{Auer2002, Thompson1933, Kveton2019, Bian2022}.\\

However, decision-making tasks naturally arise in distributed settings, such as recommender systems and sensor networks \citep{Tekin2014, Long2011}, where there exist multiple decision-makers interacting with the environment. For example, large-scale recommender systems are often a distributed system of servers, where each server hosts a decision-maker. Every time a user arrives to one of the servers, the corresponding decision-maker chooses an item to recommend, and the user will provide a reward signal in response. Upon receiving feedback, the decision-makers can update their own knowledge on the quality of this item. Furthermore, they can share this information with neighbouring servers to improve future decision-making at all servers.\\ 

Motivated by these applications, we study an extension of the traditional multi-armed bandit model where there is a network of $m$ decision-makers who each interact with the same $k$-armed bandit environment. Each agent can communicate over the network, allowing for the possibility of collaboration to speed up the learning process.

\subsection{Related Work}
Designing provably efficient algorithms for the multi-player setting is more challenging than the single-agent setting. Previous work has extended \emph{specific} single-player algorithms to the multiplayer setting through one of two methods: gossiping or electing a leader.\\

\emph{Gossiping} is a popular technique in distributed computing that uses an iterative averaging procedure to combine information from neighbouring agents to approximate the full network information \citep{Xiao2004, Duchi2012}. \citet{Landgren2016} and \citet{Rubio2019} combine this procedure with the upper confidence bound algorithm \citep{Auer2002}. \citet{Lalitha2021} show that we can use gossiping to design a provably efficient multi-agent version of Thompson Sampling \citep{Thompson1933}. Notably, all of these algorithms are asymptotically optimal. However, the analyses of these algorithms are complex, and specific to the algorithm and variant of the gossiping procedure considered. Furthermore, the performance depends on the choice of communication matrix that the gossip procedure uses for iterative averaging. Even when the network is known, choosing this matrix is a non-trivial task, and it is common to use heuristics \citep{Xiao2004}.\\

The \emph{leader-based} methods require electing a leading agent who directs the exploration of the entire network. \citet{Bar-On2019} consider the non-stochastic setting and analyse the exponential weights algorithm with numerous leaders send their action distributions to the non-leading agents in their neighbourhood. \citet{Wang2020} elect one leading agent who performs all the exploration using a upper confidence bound algorithm. Here, the leader does not observe any information from the other agents in the network, who all play greedily with respect to the empirical means of the leader. Again, these papers analyse extensions of specific bandit algorithms. Our approach also fits into the category of leader-based approaches. However, we develop a leader-based black-box reduction where the leading agent can use any bandit algorithm for decision-making. Specifically, we build upon queuing methods commonly used to handle delayed feedback \citep{Joulani2013, Mandel2015}. However, extending these queuing methods to the multi-agents setting introduces unique theoretical challenges that do not arise in single-agent scenarios.\\

Additional works investigate variations of the multi-agent bandit problem. \citet{Szorenyi2013} study the multi-armed bandit problem in peer-to-peer networks where each agent communicates with two other agents in the network, which are chosen randomly in every round. \citet{Yang2021} and \citet{Chen2023} look into the setting where each agent on the network plays at different and possibly unknown times. \cite{madhushani2021one} study the setting where there is imperfect communication between agents. \cite{Shahrampour2017, Hossain2021} study the case where every agent has their own distribution over the rewards for each action. Perhaps the best-studied multi-agent problem is where agents receive zero or degraded rewards if they play the same action at the same time. See \citet{Boursier2024} for a recent survey on this topic.

\subsection{Contributions}
We design and analyse an algorithm that reduces cooperative stochastic multi-agent $k$-armed bandit problems to the single-agent version of the problem. Our main contributions are as follows: 
\begin{itemize}
    \item We propose a black-box reduction, QuACK, that accepts any single-agent bandit algorithm as input and immediately extends it to the multi-agent setting.
    \item Theorem \ref{theorem: group-regret} shows that we can upper bound the performance over the entire network in terms of the guarantees of the chosen single-agent bandit algorithm. Pairing our reduction with a near optimal algorithm yields a near optimal algorithm for the multi-agent $k$-armed subgaussian bandit problem, up to an additive graph-dependent quantity. These results are competitive or better than the case-by-case analyses of previous works. 
    \item Our theoretical guarantees hold under mild assumptions on the bandit environment. We require that the distribution of the reward depends only on the chosen action, and each distribution has a finite first moment. Hence, if there exists a provably efficient algorithm for a single-agent bandit problem, and this bandit problem satisfies the assumptions, our reduction guarantees that it will be provably efficient in the multi-agent setting.
    This makes developing provably efficient multi-agent algorithms for various bandit problems simple. In particular, in Section~\ref{section: instances-quack}:
    \begin{itemize}
        \item [\tiny$\bullet$] We demonstrate the simplicity by using our reduction to design multi-agent algorithms for heavy-tailed bandits and duelling bandits, which have all be considered separately in the literature  \citep{Landgren2016, Dubey2020, Raveh2024}. The resulting algorithms have comparable guarantees to those developed specifically for each setting.
        \item [\tiny$\bullet$] We develop provably efficient algorithms for new multi-agent settings, such as local differential privacy, by using an appropriate single-player algorithm \citep{Ren2020}.
    \end{itemize}
    \item Finally, we perform an experimental comparison to existing works which shows that our reduction is competitive or outperforms existing methods when paired with a comparable bandit algorithm.
\end{itemize}

\section{Problem Setting}\label{section: problem-setting}
Our paper considers multi-agent variants of stochastic multi-armed bandit problems where we have a finite set of actions: $\mathcal{A}$ such that $\abs{\mathcal{A}} = k$. We formalise the network of $m$ agents and their connections through an undirected graph $G = (V, E)$ with:
\begin{align*}
    V &= \{1, 2, \cdots, m\}\\
    E &\subseteq \left\{(v, w) \in V \times V: v \neq w\right\}.
\end{align*}
Here $v \in V$ represents an agent and $e \in E$  represents a communication channel between a pair of agents. Each agent in the network is allowed to communicate with any other agent in their \emph{neighbourhood}, which we define for all $v \in V$ as follows: 
$$
    N_{v} = \left\{w \in V: (v, w) \in E\right\}.
$$
Each round consists of every agent simultaneously playing their own action, receiving their own reward and communicating with their neighbours. Formally, for each $t \in \{1, 2, \cdots, n\}$:
\begin{itemize}
    \item Each agent $v \in V$ plays action $A_{t}^{v} \in \mathcal{A}$.
    \item Each agent $v \in V$ receives reward $X_{t}^{v} \in \mathbb{R}$.
    \item Each agent $v \in V$ communicates with $w \in N_{v}$.
\end{itemize}
Throughout, we will make use of a standard and mild assumption on the bandit environment \citep{Lattimore2020}. 
\begin{assumption}\label{assumption: bandit-environment}
    The bandit environment generates feedback that depends only on the chosen action. 
    Letting $P_{a}$ denote the reward distribution for action $a$, we assume that:
    $$
        X_{t}^{v} \,\vert\, A_{t}^{v} = a \stackrel{i.i.d.}{\sim} P_{a}\, \text{ and }\, \mu_{a} = \mathbb{E}_{X\sim P_{a}}\left[X \right] < \infty\;.
    $$
    for all $t \in \mathbb{N}$, $a\in \mathcal{A}$ and $v\in V$.
\end{assumption}
Additionally, we make a mild assumption on the graph and the communication protocol.
\begin{assumption}\label{assumption: graph-communication}
    Messages can be passed from one agent $v\in V$ to another agent $w \in V$ along a shortest path in $d_{vw}$ time steps, and the graph diameter is finite: $d = \max_{v, w\in V} d_{vw} < \infty$.
\end{assumption}

\subsection{Measuring Performance}
In the multi-agent bandit problem, agents collaborate to minimise the regret of the entire network, also known as the \emph{group regret}:
$$
    R_{G}\left(n\right) = \sum_{a \in \mathcal{A}} \Delta_{a}\,\mathbb{E}\left[\,\sum_{v = 1}^{m}T_{av}\left(n\right)\right].
$$
where the expectation is over the network of agents interacting with the bandit environment. Here, $\Delta_{a} = \mu_{\star} - \mu_{a}$ is the sub-optimality gap of action $a$ where $\mu_{\star} = \max \mu_{a}$ denotes the expected reward of the optimal action, and 
\begin{equation}\label{equation: real-plays}
    T_{av}\left(n\right) = \sum_{t = 1}^{n} 1\left\{A_{t}^{v} = a\right\}
\end{equation}
is the number of times agent $v \in V$ chooses action $a \in \mathcal{A}$ over the course of $n$ rounds. From standard arguments, we can deduce a minimax lower bound on the group regret when all reward distributions are Gaussian \citep{Lattimore2020}: 
$$
    R_{G}\left(n\right) \geq \frac{1}{27} \sqrt{m n \left(k - 1\right)}.
$$
This lower bound holds for cases where each agent can communicate immediately with any other agent on the graph. Hence, it may not be tight for graphs with particular structure. However, it does demonstrate that information sharing can be used to improve the group performance in the worst-case. 

\section{A Black-Box Reduction}\label{section: black-box-reduction}
We present QuACK, \emph{a queuing algorithm for cooperative $k$-armed bandits}. QuACK is a multipurpose black-box reduction that can be paired with any single-agent bandit algorithm to extend it to the multi-agent case.\\


QuACK requires the index of the \emph{leader} on the network and a bandit algorithm $\pi$ as input. The leader will use $\pi$ to select actions for the whole network. The leader begins by initialising an empty queue for each action. These queues will store the rewards that the other agents observe from the bandit environment and have passed to the leader. Intuitively, the leader can use the reward samples in these queues instead of playing actions in the real environment. Specifically, in each round, the leader observes the action suggested by $\pi$. If there is a reward in the queue for that action, the leader will take the reward from the queue and pass this to $\pi$. Importantly, the leader does not play this action, but instead uses the reward from the queue to update $\pi$. This continues until $\pi$ suggests playing an action whose queue is empty. In this case, the leader plays the action and receives a reward from the environment. The leader then communicates this decision to the other, \emph{follower}, agents. Once a follower receives instruction from the leader (note that this can take some time since the message needs to travel over the network), they will play the instructed action, and pass the reward back to the leader (which can also take some time). Once the leader receives rewards from their followers, they will place the rewards in the corresponding queues, and begin the process again.\\

Algorithm \ref{algorithm: our-algorithm} presents the pseudo-code for QuACK. For clarity, it is presented assuming that the leader is given as input and one shortest path between each follower $w \in V$ and the chosen leader $v\in V$ has been previously identified. Appendix \ref{appendix: additional-details} explains how we can elect a leader in a distributed manner and Section \ref{section: message-passing} suggests a well-known distributed algorithm for computing the paths once we have the leader. 

\begin{algorithm}[t]\caption{QuACK}\label{algorithm: our-algorithm}
    \begin{algorithmic}
        \STATE Input: index of the leader $v \in V$
        \STATE Input for Leader: bandit algorithm $\pi$
        \STATE Initialisation for Leader: $Q_{a} = \emptyset$  for all $a\in\mathcal{A}$
        \FOR{$t \in \{1, 2, \cdots, n\}$}
            \STATE \textbf{Leader ($v$)}
            \bindent
            \STATE Receive messages via shortest path:
            $$\mathcal{M}_{t} = \{(A_{t - d_{vw}}^{w}, X_{t - d_{vw}}^{w})\}_{w \neq v}
            $$
            \STATE Append $Q_{a} = Q_{a}\cup\{(a, x)$ for $(a, x) \in \mathcal{M}_t\}$
            \vspace{0.1cm}
            \STATE Select $a$ with $\pi$
            \WHILE{$Q_{a} \neq \emptyset$}
                \STATE Remove $x$ from $Q_{a}$
                \STATE Update $\pi$ with $(a, x)$
                \STATE Select $a$ with $\pi$
            \ENDWHILE
        \STATE Play $A_{t}^{v} = a$
        \STATE Observe $X_{t}^{v} \sim P_{A_{t}^{v}}$
        \STATE Send $A_{t}^{v}$ to all $w \neq v$ via shortest path
        \vspace{0.1cm}
        \eindent
        \STATE \textbf{Follower ($w\neq v$)}
        \bindent
        \IF{$t>d_{vw}$}
            \STATE Receive $A_{t - d_{vw}}^{v}$ from $v$ via shortest path
            \STATE Play $A_{t}^{w} = A_{t - d_{vw}}^{v}$
        \ELSE
            \STATE Play $A_{t}^{w}$ uniformly at random
        \ENDIF
        \STATE Observe $X_{t}^{w} \sim P_{A_{t}^{w}}$
        \STATE Send $(A_{t}^{w}, X_{t}^{w})$ to $v$ via shortest path.
        \eindent
        \ENDFOR
    \end{algorithmic}
\end{algorithm}

\subsection{Message-Passing Protocol}\label{section: message-passing}
The message a follower $w \neq v$ sends to the leader $v$ in round $t$ contains only the action and the reward: 
$$
    (A_{t}^{w}, X_{t}^{w}).
$$
Sending such simple messages is possible since there is a known fixed route between the leader and any follower. Computing these routes can be done in a distributed manner before the start of the interaction. Specifically, this amounts to constructing the shortest path tree of the graph which can be done by solving the single-source shortest-path problem before interacting with the bandit environment \citep{Ahuja1993}. Figure \ref{figure: shortest-path-tree} illustrates the shortest path tree for the grid graph, where the black vertex represents the leader.

\begin{figure}[h!]
    \centering
    \begin{tikzpicture}
        \node[circle, draw, fill = white] (zero) at (0, 0) {};
        \node[circle, draw, fill = gray] (one) at (1, 0) {};
        \node[circle, draw, fill = white] (two) at (2, 0) {};
        
        \node[circle, draw, fill = gray] (three) at (0, 1) {};
        \node[circle, draw, fill = black] (four) at (1, 1) {};
        \node[circle, draw, fill = gray] (five) at (2, 1) {};

        \node[circle, draw, fill = white] (six) at (0, 2) {};
        \node[circle, draw, fill = gray] (seven) at (1, 2) {};
        \node[circle, draw, fill = white] (eight) at (2, 2) {};

        \draw (zero) -- (one);
        \draw (one) -- (two);
        \draw (three) -- (four);
        \draw (four) -- (five);
        \draw (six) -- (seven);
        \draw (seven) -- (eight);

        \draw (zero) -- (three);
        \draw (one) -- (four);
        \draw (two) -- (five);
        \draw (three) -- (six);
        \draw (four) -- (seven);
        \draw (five) -- (eight);
    \end{tikzpicture}
    \hspace{1cm}
    \begin{tikzpicture}
      \node[circle, draw, fill=black] (root) at (0,0) {};
    
      \node[circle, draw, fill=gray] (child1) at (-6/4,-1) {};
      \node[circle, draw, fill=gray] (child2) at (-1/2,-1) {};
      \node[circle, draw, fill=gray] (child3) at (1/2,-1) {};
      \node[circle, draw, fill=gray] (child4) at (6/4,-1) {};
    
      \node[circle, draw, fill=white] (grandchild1) at (-6/4,-2) {};
      \node[circle, draw, fill=white] (grandchild2) at (-1/2,-2) {};
      \node[circle, draw, fill=white] (grandchild3) at (1/2,-2) {};
      \node[circle, draw, fill=white] (grandchild4) at (6/4,-2) {};
    
      \draw (root) -- (child1);
      \draw (root) -- (child2);
      \draw (root) -- (child3);
      \draw (root) -- (child4);
    
      \draw (child1) -- (grandchild1);
      \draw (child2) -- (grandchild2);
      \draw (child3) -- (grandchild3);
      \draw (child4) -- (grandchild4);
    \end{tikzpicture}
    \vspace{0.3in}
    \caption{Grid Graph and its Shortest Path Tree.}
    \label{figure: shortest-path-tree}
\end{figure}
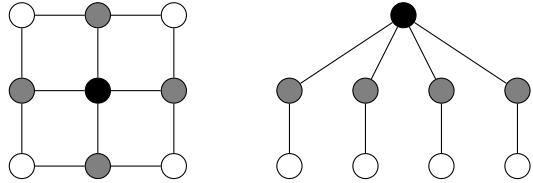

The Distributed Bellman-Ford algorithm can solve the single-source shortest-path problem exactly and it does so using approximately $m$ iterations, whilst only requiring that each agent knows their neighbours and has a unique identifier \citep{Ford1956, Bellman1958, Moore1959, Elkin2020}. Appendix \ref{appendix: additional-details} presents full details of this algorithm for completeness. \\

Sending instructions from the leader to the followers using the shortest path tree is straightforward. The leader communicates the action they played to their children, and the children send this action to their children at the end of the next round, and so forth. Using the shortest path tree also makes passing messages to the leader straightforward. Each follower will forward messages from their children to their parent, and send their own message to their parents. Thus, in the $t$-th round, each follower $w$ will perform the following communication:
\begin{itemize}
    \item Forward $A_{t}^{w} = A_{t - d_{vw}}^{v}$ to each $z \in \mathcal{C}_{w}$ 
    \item Send $M_{t}^{v} = (A_{t}^{w}, X_{t}^{w}) \cup \{M_{t - 1}^{z}\}_{z \in \mathcal{C}_{w}}$ to $\mathcal{P}_{z}$
\end{itemize}
Here, $\mathcal{C}_{w}$ and $\mathcal{P}_{w}$ denote the children and the parent of agent $w$ on the shortest path tree, respectively. The leader $v$ only needs to send the action they choose in each round to their children, which corresponds to just the first item the above list. Thus, the communication complexity, in terms of the number of messages sent in each round, for each agent $w$ is upper bounded by the diameter of the graph.\\

Note that we could use any other message-passing protocol that satisfies Assumption \ref{assumption: graph-communication}. For example, each agent could forward everything they have been sent in the past $m$ rounds. However, in this case, all messages $M_{t}^{w}$ would have to contain additional information about who originally sent the message and the time the message was sent, e.g. $w$ and $t$. This additional information is required so that the followers can identify the most recent action sent to them by the leader. Furthermore, the leader would have to check that $M_{t}^{w}$ is not added to the queue more than once. 

\subsection{Theoretical Analysis}
We now analyse QuACK (Algorithm~\ref{algorithm: our-algorithm}) with an arbitrary single player algorithm $\pi$. We will see that the group regret of QuACK depends on the single-agent regret of $\pi$ and some graph-dependent quantities. \\

We define the regret of the single-agent bandit algorithm $\pi$ over an $n$-round interaction with the single-player environment as follows: 
$$
    S_{\pi}\left(n\right) = \sum_{a \neq \star} \Delta_{a} \,\mathbb{E}_{P}\left[\sum_{t = 1}^{n} 1\{a_{t} = a\}\right]
$$
where $a_t$ denotes the action taken at time t when following $\pi$ and $\mathbb{E}_P[\cdot]$ denotes expectation with respect to the interaction with the single-player environment.\\

We will show that, by maintaining the queues, the leader creates a simulated single-player version of the environment for the bandit algorithm. For this, we need to define the number of times the single-agent bandit algorithm $\pi$ chooses action $a \in \mathcal{A}$ in totality. This counter includes the number of times the leader, who operates this single-player algorithm, chooses this action in the bandit environment, as well as the number of samples taken from the queue. Letting $a_{s}$ denote the $s$-th action chosen by the bandit algorithm allows us to define this counter: 
\begin{equation}\label{equation: pseudo-plays}
    T_{a}'\left(s_{n}\right) = \sum_{t = 1}^{n} \sum_{s = s_{t - 1} + 1}^{s_{t}} 1\left\{a_{s} = a\right\}\;.
\end{equation}
Here, $s_{t}$ represents the number of simulated rounds the bandit algorithm has completed by the end of the $t$-th round.
This includes the action chosen by the leader in this round, e.g. $A_{t}^{v} = a_{s_{t}}$. From Section \ref{section: message-passing}, we know that the leader appends each message to the queue exactly once, which implies that: $s_{n} \leq mn$.\\

Lemma \ref{lemma: simulation} shows that the bandit algorithm is operating in a simulated version of a single-player environment. 
\begin{lemma}\label{lemma: simulation}
    Under Assumption \ref{assumption: bandit-environment}, QuACK guarantees that, for all $n$:
    \begin{align*}
        S_{\pi}\left(n\right) = \sum_{a \neq \star} \Delta_{a}\, 
        \mathbb{E}\left[T_{a}'\left(n\right)\right]\;.
    \end{align*}
\end{lemma}
\begin{proof}
    Under Assumption \ref{assumption: bandit-environment}, the rewards for each action-agent pair are independent and identically distributed random variables. Properties of exchangeable random variables guarantee that the sequence of rewards the leading agent feeds to the bandit algorithm has the same distribution as in the single-player environment. Therefore, the expectation in the definition of the single-player regret and in the multi-agent setting are the same. See Appendix \ref{appendix: proof-simulation} for a formal proof.
\end{proof}
Recall that the group regret is defined in terms of times the network of agents plays each action. Therefore, Lemma \ref{lemma: utilisation} relates  this quantity to the number of times the single-agent bandit algorithm plays each action.
\begin{lemma}\label{lemma: utilisation}
    Running QuACK with an arbitrary $v\in V$ as the leader guarantees that:
    $$
        \sum_{w = 1}^{m} T_{aw}\left(t - d_{vw}\right) \leq T_{a}'\left(s_{t}\right) +  2 \sum_{w = 1}^{m} d_{vw}.
    $$
    for all $t \in \mathbb{N}$.
\end{lemma}
\begin{proof}
    See Appendix \ref{appendix: proof-utilisation}
\end{proof}

Lemmas \ref{lemma: simulation} and \ref{lemma: utilisation} allow us to upper bound the group regret by the regret of the single-player bandit algorithm plus an additive graph-dependent quantity.
\begin{theorem}\label{theorem: group-regret}
    Under Assumptions \ref{assumption: bandit-environment} and \ref{assumption: graph-communication}, the group regret of QuACK run with single-player bandit algorithm $\pi$ and leading agent $v \in V$ is bounded by
    \begin{align*}
        R_{G}\left(n\right) \leq S_{\pi}\left(mn\right) + \left(\,3 \cdot \sum_{w = 1}^{m}d_{vw}\right)\sum_{a = 1}^{k} \Delta_{a}.
    \end{align*}
\end{theorem}
\begin{proof}
Firstly, we upper bound the group plays for action $a \in \mathcal{A}$ at the end of the final round as:
\begin{align}
    \sum_{w = 1}^{m} T_{aw}\left(n\right)
    &\leq \sum_{w = 1}^{m} T_{aw}\left(n - d_{vw}\right) + \sum_{w \neq v} d_{vw}. \label{equation: leader-follow-decomp-main}
\end{align}
Combining Equation \eqref{equation: leader-follow-decomp-main} with Lemma \ref{lemma: utilisation} and $s_{n} < mn$ allows us to upper bound the number of times the entire network plays action $a$ by the number of times $\pi$ has played this action:
\begin{align}
    \sum_{w = 1}^{m} T_{aw}\left(n\right) 
    &\leq \sum_{w = 1}^{m} T_{aw}\left(n - d_{vw}\right) + \sum_{w \neq v} d_{vw}\notag\\
    &\leq T_{a}'\left(mn\right) + 3 \sum_{w \neq v} d_{vw} \label{equation: group-to-pseudo-for-regret}
\end{align}
Plugging Equation \eqref{equation: group-to-pseudo-for-regret} into the definition of the group regret gives us:
\begin{align*}
    R_{G}\left(n\right)
    &= \sum_{a \neq \star} \Delta_{a}\,\mathbb{E}\left[\,\sum_{w = 1}^{m} T_{aw}\left(n\right)\right]\\
    &\leq \sum_{a \neq \star} \Delta_{a}\,\mathbb{E}\left[T_{a}'\left(m n\right)\right] + \left[3 \sum_{w \neq v} d_{vw}\right] \sum_{a \neq \star} \Delta_{a}\\
    &= S_{\pi}\left(mn\right) + \left[3 \sum_{w \neq v} d_{vw}\right] \sum_{a \neq \star} \Delta_{a}
\end{align*}
where the final line follows from Lemma \ref{lemma: simulation} and the definition of regret in the single-player setting. See Appendix \ref{appendix: proof-group-regret} for further details.
\end{proof}

Theorem \ref{theorem: group-regret} suggests that we should pick: 
$$
    v_{\star} = \argmin_{v \in V} \sum_{w = 1}^{m} d_{vw},
$$
as the leading agent. This is intuitive, as minimising the sum of the shortest paths will minimise the total delay in sending and receiving instructions and feedback from the followers. Appendix \ref{appendix: additional-details} presents an approach for finding this agent in a distributed manner. 
Nevertheless, Assumption \ref{assumption: graph-communication} guarantees that the worst-case graph-dependence, regardless of the chosen leading agent, is given by: 
$$\sum_{w = 1}^{m} d_{vw} \leq d \left(m - 1\right).$$ 

It is important to highlight that Theorem \ref{theorem: group-regret} holds for any choice of bandit algorithm and makes weak assumptions on the rewards. This is in contrast to existing works, where it is common to analyse a specific bandit algorithm and make stronger assumptions on the reward distributions, such as subgaussian tails. In Section \ref{section: instances-quack} we show that the guarantees for QuACK are competitive with the case-specific analyses that exist in the literature, despite being strictly more general.

\section{Instances of QuACK}\label{section: instances-quack}
Theorem \ref{theorem: group-regret} holds under the assumption that the rewards are conditionally independent given the actions and the expected reward for each of the actions is finite. Therefore, we can apply our reduction in numerous bandit environments. Here, we present a non-exhaustive selection of bandit environments that satisfy Assumption \ref{assumption: bandit-environment}. For each, we describe how to translate single-agent algorithms to the multi-agent setting, provide bounds on their group regret, and compare to any existing multi-agent works in that setting. 
In each setting, we choose one specific single-player algorithm to use with QuACK.
However, any other provably efficient single-player algorithm could be combined with our reduction to get similar results.

\subsection{Subgaussian Bandits}
Traditionally, the multi-armed bandit problem consists of playing an action and receiving a reward drawn independently from a $\sigma$-subgaussian distribution that depends only on the chosen action \citep{Lattimore2020}. 
Thus, Assumption \ref{assumption: bandit-environment} holds. Corollary \ref{corollary: quack-ucb} presents the group regret bound for our algorithm when the leader uses the classic upper confidence bound algorithm \citep{Auer2002}.
\begin{corollary}\label{corollary: quack-ucb}
    Running QuACK with an arbitrary leader following the upper confidence bound algorithm with confidence parameter $\delta = 1/(m n)^{2}$ guarantees:
    $$
    R_{G}\left(n\right) \leq \sum_{a \neq \star}\frac{16 \sigma^{2}\ln\left(mn\right)}{\Delta_{a}} + \left(3 + 3\sum_{w = 1}^{m} d_{vw}\right) \sum_{a \neq \star} \Delta_{a}.
    $$
\end{corollary}
\begin{proof}
    Theorem 7.1 of \citet{Lattimore2020} provides an upper bound on the regret of the upper confidence bound algorithm over a $T$-round interaction with the bandit environment:
    \begin{equation}\label{equation: ucb-upper-bound}
        S_{\pi}\left(T\right) \leq \sum_{a \neq \star}\frac{16 \ln\left(T\right)}{\Delta_{a}} + 3\sum_{a \neq \star} \Delta_{a}.
    \end{equation}
    Combining this upper bound with Theorem \ref{theorem: group-regret} and taking $T=mn$ gives the result.
\end{proof}
\citet{Rubio2019} analyse a multi-agent extension of the upper confidence bound algorithm and offer the best theoretical guarantees in the subgaussian setting. They obtain the following bound on the group regret:
$$
    \sum_{a \neq \star} \frac{16 \eta\left(1 + \frac{\epsilon}{2}\right) \sigma^{2} \ln\left(mn\right)}{\Delta_{a}} +\left(C_{G} + m + 4\right)\sum_{a \neq \star} \Delta_{a}
$$
Here, $\eta > 1$ and $\epsilon \in (0, 1)$ are hyperparameters and $C_G$ is a graph-dependent quantity they define as:
$$
    C_{G} = \left\lceil \frac{6m \ln\left(\frac{2m}{\epsilon}\right)}{\sqrt{2\ln\abs{\lambda}^{-1}}} \right\rceil
$$
where $\lambda$ is the second largest singular value of the communication matrix they use for the gossiping procedure. Corollary \ref{corollary: quack-ucb} shows a strict improvement in the leading-order term. Comparing the graph-dependence is a little more challenging and requires specifying the communication matrix. Following their recommendations, we can show that for regular graphs, our graph-dependence is smaller than a function of theirs: 
$$
    3 \sum_{w = 1}^{m} d_{vw} < \frac{C_{G}}{\sqrt{2\ln \abs{\lambda}^{-1}}}
$$
which suggests the dependence in our bounds is better in regular graphs with a small spectral gap. Additionally, we can prove a strict improvement on the graph-dependence for specific graphs, such as the star network. 
The amount of per-round communication of the two methods is similar.
\citet{Rubio2019} have each agent send two vectors of length $k$ to each neighbour. Conversely, QuACK sends an action-reward pair to a subset of the neighbours. See Appendix \ref{appendix: comparison} for further details. 

\subsection{Heavy-Tailed Bandits}
\citet{Bubeck2013} introduced the single-agent heavy tailed bandit problem. Here, the the subgaussian assumption on the reward distribution for each action is relaxed, and instead it is assumed that each reward distribution has finite moments of order $1 + \epsilon$ where $\epsilon > 0$ controls the heaviness of the tails. \citet{Bubeck2013} propose a generic solution for the single-agent setting that they call the \emph{robust upper confidence bound algorithm} that replaces the empirical mean with a robust estimator, such as the truncated-mean, median-of-means, or Catoni's M \citep{Bubeck2013}.\\

The only difference between this setting and the subguassian setting is the weaker assumption on the tails of the rewards.
Therefore, Assumption \ref{assumption: bandit-environment} holds and the QuACK algorithm can be used for the multi-agent heavy-tailed bandit problem with the following regret guarantees.
\begin{corollary} \label{corollary: quack-heavy}
    Running QuACK with an arbitrary leader and the truncated-mean robust upper confidence bound algorithm guarantees: 
    $$
        R_{G}\left(n\right) \leq \mathcal{O}\left(\sum_{a \neq \star} \frac{\sigma^{\frac{1}{\epsilon}} \ln\left(mn\right) }{\Delta_{a}^{\frac{1}{\epsilon}}} + \left(\sum_{w = 1}^{m} d_{vw}\right)\sum_{a \neq \star}\Delta_{a}\right)
    $$
    where $\mathbb{E}_{X \sim P_{a}} \abs{X}^{1 + \epsilon} \leq \sigma$ for all $a \in\mathcal{A}$.
\end{corollary}
\citet{Dubey2020} study the multi-agent heavy-tailed bandit problem and devise two algorithms based on the truncated-mean estimator. Their best algorithm achieves a group regret bound with a similar form to Corollary~\ref{corollary: quack-heavy}, except the leading order term involving $n$ is scaled by a graph-dependent quantity. This graph-dependent term is the independence number of the $\gamma$-th power of the graph, where $\gamma$ is a hyperparameter of their algorithm. Thus, our bound is strictly better when this quantity is greater than one, which occurs when $\gamma < d$. Since they also employ a leader-based algorithm that uses a message passing procedure, the additive terms in their bound and their per-round communication are comparable. See Appendix \ref{appendix: comparison} for further details. 

\subsection{Duelling Bandits}
\citet{Yue2009a} introduced the single-agent duelling bandit problem. Here, feedback from the environment is the outcome of a duel between two actions. Modifying Algorithm \ref{algorithm: our-algorithm} for the multi-agent duelling bandit problem amounts to maintaining a queue for each pair of actions: 
$\mathcal{A}_{+} = \mathcal{A}\times \mathcal{A}$. 
Now, each agent will communicate the pair of actions and the outcome of the duel to their neighbours.\\

Assumption \ref{assumption: bandit-environment} holds since the outcomes of the duels are assumed to depend only on the pair of actions chosen in each round. Recall that we define the group regret relative to a fixed optimal action. Therefore, Theorem \ref{theorem: group-regret} applies to solution concepts, such as the Condorcet, Copeland, and Borda winners \citep{Zoghi2014, Zoghi2015, Jamieson2015}. For illustrative purposes, we present an upper bound on the group regret when the leader uses the relative upper confidence bound algorithm, which is designed under the assumption that the Condorcet winner exists \citep{Zoghi2014}.

\begin{corollary}\label{corollary: quack-duelling}
    Running QuACK with an arbitrary leader following the relative upper confidence bound algorithm of \citet{Zoghi2014} guarantees:
    $$
        R_{G}\left(n\right)
        \leq \mathcal{O}\left(\sum_{a \neq \star} \frac{\ln\left(mn\right)}{\tilde \Delta_{a}} + \left(\sum_{w = 1}^{m} d_{vw}\right)\sum_{a \neq \star}\tilde \Delta_{a}\right)
    $$
    where $\tilde \Delta_{a} = P\left(a_{\star} \succ a\right) - \nicefrac{1}{2}$ denotes the sub-optimality gap of action $a$ in the the duelling bandit setting.
\end{corollary}

\citet{Raveh2024} develop several algorithms for the multi-agent duelling bandit problem. Their best algorithm builds on the relative upper confidence bound algorithm and achieves a similar guarantee. However, the leading order term involving $n$ is scaled by the clique covering number of the $\gamma$-th power of the graph, where $\gamma$ is a hyperparameter of their algorithm. Thus, we obtain a strict improvement whenever this quantity is greater than one, which occurs when $\gamma < d$. Since their algorithm employs a message passing procedure, the additive terms in their bound and the per-round communication are comparable to ours. See Appendix \ref{appendix: comparison} for further details. 

\subsection{Local Differential Privacy}
Privacy is a topic of growing interest in the machine learning community and seeks to protect private information within datasets. \citet{Ren2020} study \emph{local differential privacy} (LDP) in the bandit setting, where the environment uses a mechanism to make the rewards private before sending them to the agent. Formally, in the multi-agent setting, the only difference is that the rewards each agent receives are privatised: 
$$
    X_{t}^{v} = f_{\epsilon}\left(X \right) \,\text{ for }\, X \stackrel{i.i.d.}{\sim} P_{A_{t}^{v}}
$$
where $f_{\epsilon}$ denotes the privatising mechanism that ensures $\epsilon$-LDP of the rewards before handing them to the agent. Procedures for privatising rewards include adding Bernoulli or Laplace noise to the rewards. Proving that the mechanism provides the specified level of privacy is independent of the bandit algorithm, e.g. Lemmas 2 and 5 of \citet{Ren2020}. Therefore, we can directly use these mechanisms in the multi-agent setting and guarantee that the rewards are $\epsilon$-LDP. Currently, no multi-agent algorithms for this problem setting exist. However, Assumption \ref{assumption: bandit-environment} is satisfied in this setting because the curator adds $i.i.d.$ noise to the rewards. Therefore, we develop the first provably efficient algorithm for multi-agent bandits with local differential privacy by plugging an existing single agent algorithm into QuACK.
\begin{corollary}
    Running QuACK with an arbitrary leader following LDP-UCB-L of \citet{Ren2020} is $\epsilon$-LDP and guarantees:
    $$
        R_{G}\left(n\right)
        \leq \mathcal{O}\left(\sum_{a \neq \star}
        \frac{1}{\epsilon^{2}}\frac{\ln\left(mn\right)}{\Delta_{a}} + \left(\sum_{w = 1}^{m} d_{vw}\right)\sum_{a \neq \star} \Delta_{a}\right). 
    $$
\end{corollary}

\section{Experimental Results}\label{section: experimental-results}
We compare the performance of QuACK (Algorithm \ref{algorithm: our-algorithm}) initialised with UCB against coop-UCB \citep{Landgren2016} and DDUCB \citep{Rubio2019}. Additionally, we initialise QuACK with Thompson Sampling to compare against Dec-TS \citep{Lalitha2021}. \\

Our experiments consider a simple bandit environment with ten actions where each reward distribution is Bernoulli with $\mu_{1} = 0.5$ and $\mu_{a} = 0.45$ otherwise. Following existing works, we conduct our experiments on several standard graph structures, namely cycle and grid graphs. Additionally, we investigate the star graph to verify our theoretical insights. All results are averaged over $100$ independent runs and we represent uncertainty by shading between the $2.5$-th and $97.5$-th quantiles.\\ 

\begin{figure*}[h!]
    \centering
    \includegraphics[]{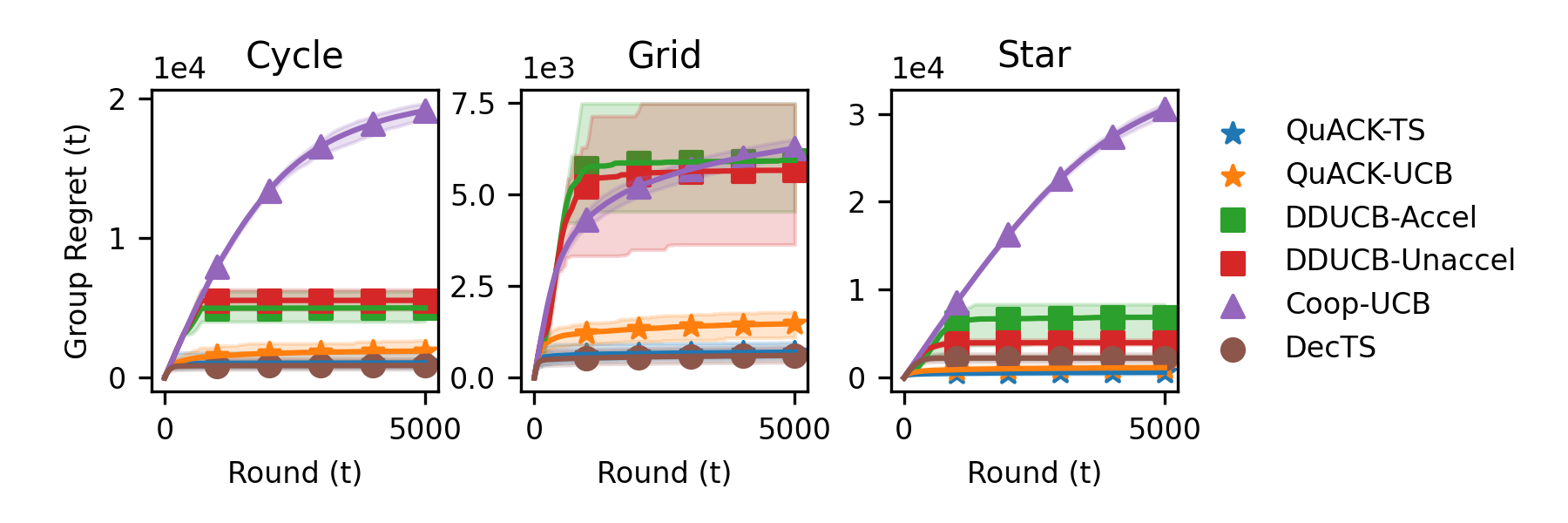}
    \caption{Group Regret for a Network of $196$ Agents.}
    \label{figure: group-regret-196}
\end{figure*}

Figure \ref{figure: group-regret-196} compares the group regret of each algorithm across various graph structures. For UCB, QuACK significantly outperforms existing methods. These results support the claims we made in Section \ref{section: instances-quack} based on our theoretical results. For Thompson Sampling, we observe that our algorithm is competitive with existing work on the cycle and grid structures, and performs significantly better on the star structure, also supporting our theoretical claims. Appendix \ref{appendix: additional-experiments} displays similar results for additional experimental settings and provides further details on the settings.

\section{Conclusion}
This paper proposes a generic black-box algorithm that can extend any single-agent bandit algorithm to the multi-agent setting. Under mild assumptions on the bandit environment, we show that our black-box approach immediately transfers the theoretical guarantees of the chosen bandit algorithm from the single agent setting to the multi-agent setting.\\

In the subgaussian setting, we can pair our algorithm with any nearly minimax optimal single agent algorithm and match the minimax lower bound in the multi-agent setting, up to an additive graph-dependent quantity. We suspect that the lower bound is loose because it fails to capture the structure of the graph and the communication delays. Proving graph-dependent lower bounds is an interesting open problem.\\ 

Additionally, we assumed that each agent can communicate real-valued messages to each neighbour at the end of every round with arbitrary precision. However, some practical applications impose constraints on the amount of communication in terms of frequency or the number of bits. Relaxing these assumptions are an important next step for this line of research.

\bibliographystyle{plainnat}
\bibliography{main}
\newpage
\begin{appendix}
\section{Implementation Details}\label{appendix: additional-details}
Section \ref{section: black-box-reduction} presents QuACK assuming that we have elected a leader and have computed a shortest path tree for our message passing protocol. Here, we describe how we can perform these steps in a distributed manner using well-known algorithms.

\subsection{Shortest Path Tree}\label{appendix: shortest-path-tree}
Constructing the shortest path tree rooted at a fixed vertex amounts to solving the single-source shortest path problem. Traditionally, we can solve this problem in non-distributed systems using the Bellman-Ford algorithm \citep{Ford1956, Bellman1958}. However, using this algorithm would require the existence of an agent who possesses knowledge of the entire network topology. Therefore, we will present the distributed variant of this algorithm \citep{Bertsekas1992}. Here, each agent knows only their neighbours.\\

Algorithm \ref{algorithm: distributed-bellman-ford} presents the pseudo-code for the Distributed Bellman-Ford algorithm. Briefly, this algorithm takes a source agent as input, and each agent keeps track of their distances from the source and their parent on the shortest path tree. 

\begin{algorithm}[h]\caption{Distributed Bellman-Ford ($\texttt{DBF}$)}
\label{algorithm: distributed-bellman-ford}
    \begin{algorithmic}
        \vspace{1mm}
        \STATE \textbf{Input.} Index of a source agent $v$ 
        \STATE Set $d_{v}^{1} = 0$ and $d_{w}^{1} = \infty$ for $w \neq v$
        \STATE Set $\mathcal{P}_{w} = \emptyset$ for each $w \in V$ to initialise their parent.
        \FOR{$t \in \{1, 2, \cdots m - 1\}$}
            \STATE Each $w \in V$ sends $d_{w}^{t}$ to $z \in N_{w}$
            \STATE Each $w \in V$ receives $d_{z}^{t}$ from $z \in N_{w}$
            \STATE Each $w \in V$ updates their distance from the source
            $$
                d_{w}^{t + 1} = \min_{z \in N_{w}\cup\{w\}} \left\{d_{z}^{t} + 1\{z \neq w\}\right\}
            $$
            \STATE Each $w \in V$ updates their parent:
            $$
                \mathcal{P}_{w} = 
                \begin{cases}
                    \argmin_{z \in N_{w}} d_{z}^{t} + 1 & \text{ if } d_{w}^{t + 1} < d_{w}^{t}\\
                    \mathcal{P}_{w} & \text{ if } d_{w}^{t + 1} = d_{w}^{t}
                \end{cases}
            $$
        \ENDFOR
        \STATE Each $w \in V$ sends their identifier $w$ to their parent $\mathcal{P}_{w}$.
        \STATE Each $w \in V$ receives the identifiers from their children and creates their set of children:
            $$
                \mathcal{C}_{w} = \{z \in V: \mathcal{P}_{z} = w\}
            $$
    \end{algorithmic}
\end{algorithm}

After Algorithm \ref{algorithm: distributed-bellman-ford} terminates, each agent will know their children and parent on the shortest path tree rooted at the index of the input agent. Indeed, these are the quantities we require to implement the message-passing scheme described in Section \ref{section: message-passing}.

\subsection{Leader Election}\label{appendix: leader-election}

Leader Election (LE) is a well-known problem in distributed computing \citep{Santoro2006, Fokkink2013}. Here, we want the network to unanimously agree upon a single agent as a leader. Generally, there are impossibility results for anonymous networks \citep{Fokkink2013}. Therefore, we will assume that each agent has a unique identifier. \citet{Wang2020} devise a simple leader election scheme that requires at most $d + 1$ iterations, where $d$ is the diameter of the graph. Algorithm \ref{algorithm: leader-election} presents a slight modification of their procedure that will ease the exposition of finding the best leader. 

\begin{algorithm}[h]\caption{Leader Election ($\texttt{LE}$)}\label{algorithm: leader-election}
    \begin{algorithmic}
        \vspace{1mm}
        \STATE Each $v \in V$ sets $f_{v}$ arbitrarily.
        \bindent
        \STATE Section \ref{appendix: optimising-leader} has each agent set $f_{v}$ to their sum of shortest paths.
        \eindent
        \STATE Each $v \in V$ saves their initial value $I_{v} = (v, f_{v})$
        \STATE Each $v \in V$ creates their state value $S_{v}^{1} = (v, f_{v})$
        \FOR{$t \in \{1, 2, \cdots, m\}$}
            \STATE Each $v \in V$ receives $S_{w}^{t}$ from $w \in N_{v}$
            \IF{$\exists w \in N_{v}$ such that $f_{w} < f_{v}$}
                \STATE Update $S_{v}^{t + 1} = (w, f_{w})$
            \ELSIF{$\exists w \in N_{v}$ such that $f_{w} = f_{v}$}
                \STATE Update $S_{v}^{t + 1} = (\min\{w, v\}, f_{v})$
            \ELSE
                \STATE $S_{v}^{t + 1} = S_{v}^{t}$
            \ENDIF            
        \ENDFOR
        \IF{$I_{v} = S_{v}^{m + 1}$}
            \STATE Agent $v$ is the leader.
        \ENDIF
    \end{algorithmic}
\end{algorithm}
Running Algorithm \ref{algorithm: leader-election} requires at most $m$ iterations because every agent will have seen the identifier of every other agent in the graph at the end of this iteration. Furthermore, Algorithm \ref{algorithm: leader-election} will always terminate with exactly one leader. To see this, we can consider two possible cases.
\begin{itemize}
    \item \textit{Unique Minimum}. Here, there exists exactly one agent $v$ who possesses the smallest $f_{v}$ value. Therefore, they never get to update their state and become the leader.
    \item \textit{Non-Unique Minimum}. Let $f = \min_{v \in V} f_{v}$ and suppose multiple agents attain this value. Algorithm \ref{algorithm: leader-election} implements an index-based tie-breaking rule to handle this case. Notably, agents with $f_{v} = f$ will eventually receive an $f_{w} = f$ and they will proceed to check if their index is minimal. Since each agent has a unique index, there will be exactly one agent who possesses $f_{v} = f$ with the minimum index. This agent will never update their state and will become the leader. 
\end{itemize}

\subsection{Optimising the Leader}\label{appendix: optimising-leader}
\begin{algorithm}[h]\caption{Best Leader}\label{algorithm: best-leader}
    \begin{algorithmic}
        \vspace{1mm}
        \FOR{$v \in V = \{1, 2, \cdots, m\}$}
            \STATE \textbf{Compute Sum of Shortest Paths}
            \STATE Execute $\texttt{DBF}(v)$
            \bindent
            \STATE Each agent now knows $d_{vw}$ (distance from source $v$)
            \STATE Each agent now knows $\mathcal{P}_{w}$ (their parent on the shortest path tree with source $v$)
            \eindent
            \STATE Each $w \in V$ creates message $M_{1}^{w} = (w, d_{vw})$ 
            \FOR{$t \in \{1, 2, \cdots, m\}$}
                \STATE Each $w \in V$ sends $M_{t}^{w}$ to $\mathcal{P}_{w}$
                \STATE Each $w \in V$ updates $M_{t + 1}^{w} = M_{t}^{w} \cup \{M_{t}^{z}\}_{z \in \mathcal{C}_{w}}$
            \ENDFOR
            \STATE Agent $v$ calculates $f_{v} = \sum_{w} d_{vw}$ from the messages.
        \ENDFOR

        \STATE \textbf{Use Sum of Shortest Paths as $f_{v}$ in Leader Election}
        \STATE Each agent now knows their value for $f_{v} = \sum_{w = 1}^{m} d_{vw}$
        \STATE Execute $\texttt{LE}$ to elect the leader with minimal sum of shortest paths.
    \end{algorithmic}
\end{algorithm}

Combining Sections \ref{appendix: leader-election} and \ref{appendix: shortest-path-tree} allows us to find the best-positioned leading agent in the network according to our theoretical guarantees. Algorithm \ref{algorithm: best-leader} presents the pseudo-code for finding and electing the leader.\\

Briefly, this procedure starts with each agent calling Algorithm \ref{algorithm: distributed-bellman-ford}. Afterwards, we perform several rounds of message passing so that agent $v$ can calculate the sum of shortest paths: $f_{v}$. Finally, Algorithm \ref{algorithm: leader-election} performs leader election with $f_{v}$ equal to the sum of shortest paths.

\section{Missing Proofs}\label{appendix: missing-proofs}
Throughout the main paper we deferred several proofs to the appendix and gave sketches of the proofs for the main results. Here, we present full proofs of the claims made in the main paper and we do so in chronological order. Specifically, the remainder of this section will have the following structure:
\begin{itemize}
    \item Section \ref{appendix: proof-simulation} presents a proof of Lemma \ref{lemma: simulation} which establishes the equivalence of playing in the bandit environment and the queued version of the environment created by the feedback from non-leading agents.
    \item Section \ref{appendix: proof-utilisation} presents a proof of Lemma \ref{lemma: utilisation} which tells us that the difference between the group plays and the pseudo plays of the leader is almost surely bounded by a graph-dependent constant.
    \item Section \ref{appendix: proof-group-regret} presents a full proof of Theorem \ref{theorem: group-regret} and fills in steps missing from the proof sketches given in the main paper.
\end{itemize}

\subsection{Proof of Lemma \ref{lemma: simulation}}\label{appendix: proof-simulation}
In QuACK, the leader uses rewards from the queues as well as their own observations to \emph{simulate} an environment for updating the bandit algorithm $\pi$. Recall that the rewards in the queues have been gathered by other agents in the network when interacting with the environment. Lemma \ref{lemma: simulation} relates the expected number of times the leader updates the algorithm $\pi$ for each action $a$ to the expected number of times the single-agent bandit algorithm $\pi$ would play action $a$ over an $n$-round interaction with the environment. More precisely it states that:
\begin{equation}\label{eq:lemma1appendix}
    \sum_{a \neq \star} \Delta_{a}\,\mathbb{E}_{P}\left[\,\sum_{t = 1}^{n} 1\{A_{t} = a\}\right] = \sum_{a \neq \star} \Delta_{a}\, 
        \mathbb{E}\left[T_{a}'\left(n\right)\right]\;.
\end{equation}
In Equation \eqref{eq:lemma1appendix}, the first expectation is with respect to the bandit algorithm $\pi$ interacting directly with the bandit environment. The second expectation is with respect to bandit algorithm $\pi$ interacting with the \textit{simulated} environment created by the leader. \\

For each action $a\in\mathcal{A}$, let $P_{a}$ be the distribution over possible rewards when playing action $a$ in the bandit environment, as defined in Assumption \ref{assumption: bandit-environment}. For each $a$, we also define $\tilde{P}_{a}$ which is a probability measure over the possible rewards for playing action $a$ in the simulated environment created by the leader. 
In addition let $p_{\pi}$ be the density over action-reward pairs up to time $n$ when the single player bandit  algorithm $\pi$ interacts with the environment, and let $\tilde{p}_{\pi}$ be the density over action-reward pairs in the simulated environment up to time $n$.
Following the notation of \citet{Lattimore2020}, we can write down these densities as follows:
\begin{align*}
    p_{\pi}\left(a_{1}, x_{1}, \cdots, a_{n}, x_{n}\right) &= \prod_{t = 1}^{n} \pi\left(a_{t}\,\vert\, a_{1}, x_{1}, \cdots, a_{t - 1}, x_{t - 1}\right) p_{a_{t}}\left(x_{t}\right)\\
    \tilde{p}_{\pi}\left(a_{1}, x_{1}, \cdots, a_{n}, x_{n}\right) &= \prod_{t = 1}^{n} \pi\left(a_{t}\,\vert\, a_{1}, x_{1}, \cdots, a_{t - 1}, x_{t - 1}\right) \tilde{p}_{a_{t}}\left(x_{t}\right)
\end{align*}
where $p_{a}$ and $\tilde{p}_{a}$ denote the Radon-Nikodym derivatives of $P_{a}$ and $\tilde{P}_{a}$, respectively. Using these notations we can define the two expectations in equation~\eqref{eq:lemma1appendix} more precisely, i.e. for any action $a\in\mathcal{A}$,
$$
\mathbb{E}_{P}\left[\,\sum_{t = 1}^{n} 1\{A_{t} = a\}\right] =\int_{\mathcal{H}_{n}} \sum_{t = 1}^{n}1\{a_{t} = a\}\;p_\pi(h)\,dh $$
whereas
$$
\mathbb{E}\left[T_{a}'\left(n\right)\right]=\int_{\mathcal{H}_{n}} \sum_{t = 1}^{n}1\{a_{t} = a\}\;\tilde{p}_\pi(h)\,dh
$$
where $h=(x_1,a_1,\dots x_n,a_n)\in \mathcal{H}_{n} = \left(\mathcal{A} \times \mathbb{R}\right)^{n}$ is a sequence of $n$ action-reward pairs. Under Assumption \ref{assumption: bandit-environment} we know that the rewards for each action-agent pair are independent and identically distributed random variables. Therefore, using exchangeability of independent and identically distributed random variables, we have that:
$$
    P_{a} \stackrel{d}{=} \tilde{P}_{a}
$$
The Radon-Nikodym Theorem tells us that the $p_{a}$ is unique, 
e.g. see Theorem 4.2.2 of \citet{Cohn2013}. Combining $P_{a} = \tilde{P}_{a}$ with this well-known result tells us that $p_{a}(x) = \tilde{p}_{a}(x)$  for all $x\in \mathbb{R}$. Therefore $p_{\pi}(h)=\tilde{p}_{\pi}(h)$ for all $h\in \mathcal{H}_{n} = \left(\mathcal{A} \times \mathbb{R}\right)^{n}$ and hence,
$$\int_{\mathcal{H}_{n}} f\left(h\right) p_{\pi}\left(h\right) \,dh = \int_{\mathcal{H}_{n}} f\left(h\right) \tilde{p}_{\pi}\left(h\right)\,dh
$$
for any arbitrary function $f: \mathcal{H}_{n} \rightarrow \mathbb{R}$. Choosing $f(h) = \sum_{t = 1}^{n}1\{a_{t} = a\}$ completes the proof.

\subsection{Proof of Lemma \ref{lemma: utilisation}}\label{appendix: proof-utilisation}
Define the last round that the leader vertex plays action $a$ in the bandit environment as follows: 
$$
    \tau = \max\left\{j \leq t: A_{j}^{v} = a\right\}.
$$
Recall $s_{t}$ denotes the pseudo round counter when the leading agent makes their $t$-th decision in the bandit environment. Using Equation \eqref{equation: pseudo-plays} and leveraging the fact that the round in the previous display is associated with a play of action $a$ gives us the following:
\begin{align}
    T_{av}'\left(s_{t}\right) 
    &= \sum_{s = 1}^{s_{t}} 1\left\{a_{s}^{v} = a\right\}\notag\\
    &= \sum_{s = 1}^{s_{\tau}} 1\left\{a_{s}^{v} = a\right\} + \sum_{s = s_{\tau} + 1}^{s_{t}} 1\left\{a_{s}^{v} = a\right\}\notag\\
    &= T_{av}'\left(s_{\tau}\right) + \sum_{s = s_{\tau} + 1}^{s_{t}} 1\left\{a_{s}^{v} = a\right\}\notag\\
    &\stackrel{(\star)}{=} \sum_{w = 1}^{m} T_{aw}\left(\tau - d_{vw}\right) + \sum_{s = s_{\tau} + 1}^{s_{t}} 1\left\{a_{s}^{v} = a\right\}\notag\\
    &\geq \sum_{w = 1}^{m} T_{aw}\left(\tau - d_{vw}\right) \label{equation: pseudo-to-group}
\end{align}
where $(\star)$ follows from the fact that the leader plays action $a$ in round $\tau$ so the number of times they play this action in the queue must be equivalent to the group plays of this action in the bandit environment. Now, Equation \eqref{equation: pseudo-to-group} relates the number of times the bandit algorithm has played action $a$ to the number of times the network has chosen this action. Recall that each follower will play action $a$ for the last time in round: 
$$
    \tau_{w} = \tau + d_{vw}
$$
where $\tau$ is the last round that the leader plays action $a$ in the bandit environment. Using Equation \eqref{equation: real-plays} with the definition and properties of $\tau_{w}$ allows us to get an upper bound on the number of plays for any non-leading agent $w$ as follows:
\begin{align}
    T_{aw}\left(t - d_{vw}\right)
    &= \sum_{j = 1}^{t - d_{vw}} 1\{A_{j}^{w} = a\}\notag\\
    &= \sum_{j = 1}^{\tau - d_{vw}} 1\{A_{j}^{w} = a\} + \sum_{j = \tau - d_{vw} + 1}^{t - d_{vw}} 1\{A_{j}^{w} = a\}\notag\\
    &\stackrel{(\star)}{\leq} T_{aw}\left(\tau - d_{vw}\right) + \sum_{j = \tau - d_{vw} + 1}^{\tau + d_{vw}} 1\{A_{j}^{w} = a\}\notag\\
    &\leq T_{aw}\left(\tau - d_{vw}\right) + 2d_{vw} \label{equation: follower-random-follower-notrandom}
\end{align}
Notably, $(\star)$ follows from analysing the two possible cases: 
\begin{itemize}
    \item When $t - d_{vw} \geq \tau_{w} = \tau + d_{vw}$, the inequality is actually an equality because $\tau_{w}$ is the last round follower $w$ plays this action up to and including the $t$-th round.
    \item Otherwise, $t - d_{vw} < \tau_{w} = \tau + d_{vw}$ and we are extending the summation to include more rounds which gives us the inequality. 
\end{itemize}
Summing Equation \eqref{equation: follower-random-follower-notrandom} over all agents and substituting into Equation \eqref{equation: pseudo-to-group} gives us: 
\begin{align*}
    \sum_{w = 1}^{m} T_{aw}\left(t - d_{vw}\right)
    &\leq \sum_{w = 1}^{m}T_{aw}\left(\tau - d_{vw}\right) + 2\sum_{w = 1}^{m}d_{vw} \tag{Equation \eqref{equation: follower-random-follower-notrandom}}\\
    &\leq T_{a}'\left(s_{t}\right) + 2\sum_{w = 1}^{m}d_{vw} \tag{Equation \eqref{equation: pseudo-to-group}}
\end{align*}
as required. 

\subsection{Proof of Theorem \ref{theorem: group-regret}}\label{appendix: proof-group-regret}
Within Section \ref{section: black-box-reduction}, we provided a nearly complete proof. Here, we present each step for completeness. Firstly, we can rewrite the group plays for action $a$ at the end of the final round as follows:
\begin{align}
    \sum_{w = 1}^{m} T_{aw}\left(n\right)
    &= T_{av}\left(n\right) + \sum_{w \neq v}T_{aw}\left(n\right)\notag\\
    &= T_{av}\left(n\right) + \adjustlimits  \sum_{w \neq v} \sum_{t = 1}^{n} 1\{A_{t}^{w} = a\}\notag\\
    &= T_{av}\left(n\right) + \adjustlimits  \sum_{w \neq v} \sum_{t = 1}^{n - d_{vw}} 1\{A_{t}^{w} = a\} + \sum_{w \neq v} \sum_{t = n - d_{vw} + 1}^{n } 1\{A_{t}^{w} = a\} \notag\\
    &= T_{av}\left(n\right) + \sum_{w \neq v} T_{aw}\left(n - d_{vw}\right) + \sum_{w \neq v} \sum_{t = n - d_{vw} + 1}^{n } 1\{A_{t}^{w} = a\}\notag\\
    &\leq T_{av}\left(n\right) + \sum_{w \neq v} T_{aw}\left(n - d_{vw}\right) + \sum_{w \neq v} d_{vw} \label{equation: leader-follow-decomp}
\end{align}
Combining Equation \eqref{equation: leader-follow-decomp} with Lemma \ref{lemma: utilisation} allows us to upper bound the group plays by the number of times the leader plays the action in the queued version of the environment:
\begin{align}
    \sum_{w = 1}^{m} T_{aw}\left(n\right) 
    &\leq T_{av}\left(n\right) + \sum_{w \neq v} T_{av}\left(n - d_{vw}\right) + \sum_{w \neq v} d_{vw} \tag{Equation \eqref{equation: leader-follow-decomp}}\\
    &= \sum_{w = 1}^{m} T_{aw}\left(n - d_{vw}\right) + \sum_{w = 1}^{m} d_{vw}  \tag{Since $d_{vv} = 0$}\\
    &\leq T_{av}'\left(s_{n}\right) + 3 \sum_{w \neq v} d_{vw} \label{equation: group-to-pseudo}
\end{align} 
where the final line follows from Lemma \ref{lemma: utilisation}. Plugging Equation \eqref{equation: group-to-pseudo} allows us to get an upper bound on the group regret: 
\begin{align}
    R_{G}\left(n\right)
    &= \sum_{a \neq \star} \Delta_{a}\,\mathbb{E}\left[\sum_{w = 1}^{m} T_{aw}\left(n\right)\right]\notag\\
    &\leq \sum_{a \neq \star} \Delta_{a}\,\mathbb{E}\left[T_{av}'\left(s_{n}\right) + 3 \sum_{w \neq v} d_{vw}\right] \tag{Equation \eqref{equation: group-to-pseudo}}\\
    &= \sum_{a \neq \star} \Delta_{a}\,\mathbb{E}\left[T_{av}'\left(s_{n}\right)\right] + \left[3 \sum_{w \neq v} d_{vw}\right] \sum_{a \neq \star} \Delta_{a}\notag\\
    &\leq \sum_{a \neq \star} \Delta_{a}\,\mathbb{E}\left[T_{av}'\left(m n\right)\right] + \left[3 \sum_{w \neq v} d_{vw}\right] \sum_{a \neq \star} \Delta_{a} \tag{Since $s_{n} \leq mn$}\\
    &= S_{\pi}\left(mn\right) + \left[3 \sum_{w \neq v} d_{vw}\right] \sum_{a \neq \star} \Delta_{a}
\end{align}
where the final line follows from Lemma \ref{lemma: simulation}.

\section{Regret Bound Comparison}\label{appendix: comparison}
Our theoretical results span several multi-agent bandit problems, which include: subgaussian rewards, heavy-tailed rewards, duelling bandits and bandits with local differential privacy, amongst any other bandit problems that satisfied Assumption \ref{assumption: bandit-environment}. Here, we compare our bounds to existing works in the literature. 

\subsection{Subgaussian Rewards}\label{appendix: subgaussian}
Several works design algorithms under a subgaussian assumption on the reward distributions \citep{Landgren2016, Rubio2019, Lalitha2021}. \citet{Rubio2019} show that their guarantees improve upon those shown by \citet{Landgren2016}. \citet{Lalitha2021} remark that their analysis does not provide guarantees as tight as those found in previous works \citep{Landgren2016, Rubio2019}. Thus, we focus on a comparison with DDUCB of \citet{Rubio2019}.\\

\cite{Rubio2019} use a gossiping procedure. This requires a communication matrix, $P \in \mathbb{R}^{m\times m}$. Notably, $P$ is doubly stochastic and must respect the structure of the graph, e.g. $P_{vw} = 0$ if $(v, w) \not\in E$. Throughout, we will require the eigenvalues of the $P$ and follow the authors in assuming that the eigenvalues of this matrix satisfy the following: 
$$1 = \lambda_{1}\left(P\right) > \abs{\lambda_{2}\left(P\right)} \geq \abs{\lambda_{3}\left(P\right)} \geq \cdots \geq \abs{\lambda_{m}\left(P\right)} \geq 0$$

Theorem 3.2 of \citet{Rubio2019} prove the following upper bound on the group regret of their algorithm, DDUCB:
$$
    \sum_{a \neq \star}\frac{16 \eta \left(1 + \frac{\epsilon}{2}\right) \sigma^{2}\ln\left(mn\right)}{\Delta_{a}} + \left(C_{G} + m + 4\right) \sum_{a \neq \star} \Delta_{a}
$$
Here, $\eta > 1$ and $\epsilon \in (0, \nicefrac{\eta - 1}{7(\eta + 1)}$ are tune-able hyperparameters and $C_{G}$ is a graph-dependent quantity defined as: 
$$
    C_{G} = \left\lceil \frac{6m \ln\left(\frac{2m}{\epsilon}\right)}{\sqrt{2\ln\abs{\lambda}^{-1}}} \right\rceil.
$$
Corollary \ref{corollary: quack-ucb} shows that the group regret for running QuACK with the upper confidence bound algorithm yields the following bound on the group regret:
$$
    R_{G}\left(n\right) \leq \sum_{a\neq \star}\frac{16\sigma^{2}\ln\left(mn\right)}{\Delta_{a}} + \left(3 + 3\sum_{w = 1}^{m} d_{vw}\right)\sum_{a \neq \star}\Delta_{a}
$$
Inspecting the leading order terms reveals:
$$
    16\eta\left(1 + \frac{\epsilon}{2}\right) \sigma^{2} \ln\left(mn\right) > 16 \sigma^{2} \ln\left(mn\right)
$$
where the inequality follows from the fact that $\eta > 1$ and $\epsilon > 0$. Thus, we obtain a strictly better constant multiplying the leading order term.

The additive graph dependent terms are different in each bound. Although, we remark that these terms become negligible for $n$ sufficiently large, it is still interesting to compare them.
To compare these additive graph dependent terms, we consider specific graph structures, such as regular and star graphs. 
Throughout, we will assume that the communication matrix is chosen as-per the recommendations of \citet{Rubio2019}: 
$$
    P = I - \frac{1}{1 + \max_{v \in V}\abs{N_{v}}}\left(D - M\right)
$$
where $D$ and $M$ denote the degree matrix and the adjacency matrix of the graph. Formally, $D$ and $M$ have the following element-wise definitions: 
\begin{align*}
    D_{vw} = 
    \begin{cases}
        \abs{N_{v}} & \text{ if } w = v\\
        0 & \text{ otherwise}
    \end{cases}
    \quad \text{and}\quad
    M_{vw} 
    = \begin{cases}
        1 & \text{ if } (v, w) \in E\\
        0 & \text{ otherwise}
    \end{cases}
\end{align*}

\paragraph{Regular Graphs.} Graphs are called $\delta$-regular if every agent has exactly $\delta$ neighbours, e.g. $\abs{N_{v}} = \delta$ for all $v \in V$. For these graphs, the expression for $P$ simplifies and allows us to obtain an expression for the spectral term found in the graph-dependent term of \citet{Rubio2019}: 
$$
    P = \frac{I + M}{1 + \delta} \implies \lambda \coloneqq \lambda_{2}\left(P\right) = \frac{1 + \lambda_{2}\left(M\right)}{1 + \delta}
$$
\citet{Chung1989} provide an upper bound on $d$ for regular graphs, which we can substitute into the graph-dependent term in the group regret of QuACK-UCB:
\begin{align*}
    3 + 3 \sum_{w = 1}^{m} d_{vw} 
    &\leq 3 + 3 d \left(m - 1\right)\\
    &\leq 3 + \left\lceil \frac{3\left(m - 1\right) \ln\left(m - 1\right)}{\ln\left(\frac{1 + \delta}{\abs{1 + \lambda_{2}\left(M\right)}}\right)}\right\rceil\\
    &= 3 + \left\lceil \frac{3\left(m - 1\right) \ln\left(m - 1\right)}{\ln\abs{\lambda}^{-1}}\right\rceil\\
    &= 3 + \left\lceil \frac{6\left(m - 1\right) \ln\left(m - 1\right)}{\sqrt{2\ln\abs{\lambda}^{-1}}\sqrt{2\ln\abs{\lambda}^{-1}}}\right\rceil\\
    &\stackrel{(\star)}{<} 3 + \left\lceil \frac{6\left(m - 1\right) \ln\left(\frac{2m}{\epsilon}\right)}{\sqrt{2\ln\abs{\lambda}^{-1}}\sqrt{2\ln\abs{\lambda}^{-1}}}\right\rceil\\
    &= 3 + \frac{C_{G} - 6\ln (m - 1)}{\sqrt{2\ln \abs{\lambda}^{-1}}}\\
    &< \frac{C_{G}}{\sqrt{2\ln \abs{\lambda}^{-1}}} + m + 4 
\end{align*}
where $(\star)$ follows from the fact that $\ln(2m/\epsilon) < \ln (m - 1)$ for $\epsilon \in (0, 1)$. This establishes a relationship between an upper bound on the graph-dependence of our algorithm with the graph-dependence of \citet{Rubio2019}. Generally, this suggests that the graph-dependence of our leader-based approach is smaller for $\delta$-regular graphs where the communication matrix has large values for $\abs{\lambda}$. Hence, in these cases, the graph-dependent terms in our bound on the group regret are smaller. 
   
\paragraph{Star Graphs.} The star graph consists of $m - 1$ vertices connected to a single central vertex. Here, we can exactly evaluate the spectrum of the communication matrix. Notably, the spectrum of $L \coloneqq D - M$ is known for this graph \citep{Chung1997}:
$$
    \lambda\left(L\right) =
    \begin{cases}
        m & \text{ with multiplicity } 1\\
        1 & \text{ with multiplicity } m - 2\\
        0 & \text{ with multiplicity } 1
    \end{cases}
$$
Therefore, the spectrum of the communication matrix $P$ can be found by plugging in these quantities:
$$
    \lambda\left(P\right) = 1 - \frac{\lambda\left(L\right)}{m} = 
    \begin{cases}
        0 & \text{ with multiplicity }1\\
        \frac{m - 1}{m} & \text{ with multiplicity } m - 2\\
        1 & \text{ with multiplicity } 1
    \end{cases}
$$
which gives us $\lambda = \lambda_{2}(P) = \nicefrac{m - 1}{m}$. Plugging this into their graph-dependent term yields:
\begin{align*}
    C_{G} + m + 4 
    &= \left\lceil\frac{6m\ln\left(\frac{m}{\epsilon}\right)}{\sqrt{2\ln\left(\frac{m}{m - 1}\right)}}\right\rceil + m + 4 \\
    &\geq \frac{6m \ln\left(m\right)}{\sqrt{2\ln\left(\frac{m}{m - 1}\right)}} + m + 4\\
    &> 7m + 4
\end{align*}
where the final inequality follows from the fact that $m \geq 2$ is required for the multi-agent setting. Running Algorithm \ref{algorithm: best-leader} will choose the central vertex as the leading agent and this will give make the graph-dependence in our algorithm: 
$$
    3 + \sum_{w = 1}^{m} d_{vw} = m + 2
$$
which is strictly smaller than that of \citet{Rubio2019}.

\subsection{Heavy-Tailed Rewards}\label{appendix: heavy-compare}
\citet{Dubey2020} design a multi-agent extension of the robust upper confidence bound algorithm using the truncated mean estimator. To the best of our knowledge, this is the only work addressing heavy-tailed environments in the multi-agent setting. Before stating their theoretical results, we first define the independence number of a graph: 
$$
    \alpha\left(G\right) = \max_{S \subseteq V}\left\{\abs{S}: (v, w) \not\in E \text{ for all } (v, w) \in S\right\}
$$
which is the size of the largest subset of vertices where no two vertices are adjacent. \cite{Dubey2020} develop several algorithms, and their best guarantee on the group regret up to absolute constants is given by: 
$$
    \alpha\left(G_{\gamma}\right)\sum_{a \neq \star} \frac{ \sigma^{\frac{1}{\epsilon}}\ln\left(n\right)}{\Delta_{a}^{\frac{1}{\epsilon}}} + \left( m \gamma\cdot \alpha\left(G_{\gamma}\right) + \alpha\left(G_{\gamma}\right) + m \right)\sum_{a \neq \star} \Delta_{a}
$$
Here, $\epsilon$ controls the heaviness of the tails for the reward distributions, $\gamma$ is a parameter their algorithm takes as input that governs how far each agent can communicate, and the $\gamma$-th power of the graph is defined as: 
$$
    G_{\gamma} = (V, E_{\gamma}) \; \text{ where } \; E_{\gamma} = 1\{(v, w) \in V \times V: d_{vw} \leq \gamma\}
$$
which adds an edge to the graph whenever $d_{vw} \leq \gamma$. From Corollary \ref{corollary: quack-heavy}, the group regret of running QuACK with the robust upper confidence bound using the truncated-mean estimator has the following upper bound: 
$$
    R_{G}\left(n\right) \leq \mathcal{O}\left(\sum_{a \neq \star} \frac{\sigma^{\frac{1}{\epsilon}} \ln\left(mn\right) }{\Delta_{a}^{\frac{1}{\epsilon}}} + \left(\sum_{w = 1}^{m} d_{vw}\right)\sum_{a \neq \star}\Delta_{a}\right)
$$
Thus, we obtain a smaller leading order term whenever:
$$
    n^{\alpha\left(G_{\gamma}\right) - 1} \geq m
$$
Notably, $\alpha\left(G_{\gamma}\right) \in [1, m]$ for all graphs. Therefore, our graph-dependence is smaller whenever $\alpha\left(G_{\gamma}\right) > 1$ and $m < n$, i.e. the number of agents is smaller than the number of rounds.\footnote{Note that $m > n$ implies that all multi-agent bounds on the group regret become linear in the horizon.} Specifying $\gamma \geq d$ as the input parameter will yields an independence number of $1$ in the bound \citet{Dubey2020}. However, their analysis is specific to the robust upper confidence bound algorithm and the truncated mean estimator, whereas Theorem \ref{theorem: group-regret} is strictly more general.\\

\citet{Dubey2020} have an additive graph-dependent term that displays a trade-off between the input parameter $\gamma$ and the independent number. Generally, increasing $\gamma$ will decrease the independence number and decreasing $\gamma$ will increase the independent number. Generally, we expect their additive term to be comparable to ours in many cases. For example, choosing $\gamma = d$ will minimise the leading order term in the regret bound and yields an additive graph-dependence of: 
$$
    m \gamma \cdot \alpha\left(G_{\gamma}\right) + \alpha\left(G_{\gamma}\right) + m = m d + m + 1
$$
Conversely Corollary \ref{corollary: quack-heavy} shows that our additive graph-dependence of the order: 
$$
    \sum_{w = 1}^{m} d_{vw} \leq m \left(d - 1\right)
$$ 

\subsection{Duelling Bandits}
\citet{Raveh2024} design a multi-agent extension of various duelling bandit algorithms. To the best of our knowledge, this is this only work addressing the multi-agent duelling bandit problem. Their best guarantee is for an extension of the relative upper confidence bound algorithm \citep{Zoghi2014}. Up to absolute constants, they show that this algorithm achieves an upper bound on the group regret given by:
$$
    \chi\left(G_{\gamma}\right) \sum_{a \neq \star} \frac{\ln\left(n\right)}{\tilde{\Delta}_{a}} + m^{2} k^{2}\left(1 + \ln\left(1 + \abs{N_{\star}^{\gamma}}\right) \right)\tilde{\Delta}_{\max} + \left(1 + \gamma\right) k m \tilde{\Delta}_{\max} 
$$
where $1 \leq \chi(G_{\gamma}) \leq m$ denotes the clique covering number of the $\gamma$-th power of the graph and $\gamma$ is an input parameter. From Corollary \ref{corollary: quack-duelling}, we can upper bound the group regret when running QuACK with the relative upper confidence bound algorithm: 
$$
    R_{G}\left(n\right)
    \leq \mathcal{O}\left(\sum_{a \neq \star} \frac{\ln\left(mn\right)}{\tilde \Delta_{a}} + \left(\sum_{w = 1}^{m} d_{vw}\right)\sum_{a \neq \star}\tilde \Delta_{a}\right)
$$
Comparing the leading-order term reveals that we obtain a strictly better leading order term whenever: 
$$
    n^{\chi\left(G_{\gamma}\right) - 1} \geq m
$$

Notably, $\chi\left(G_{\gamma}\right) \in [1, m]$ for all graphs. Therefore, our graph-dependence is smaller whenever $\chi\left(G_{\gamma}\right) > 1$ and $m < n$, i.e. the number of agents is smaller than the number of rounds. Specifying $\gamma \geq d$ as the input parameter will yield a clique covering number of $1$ in the bound of \citet{Raveh2024}. However, their analysis is specific to the relative upper confidence bound algorithm, whereas Theorem \ref{theorem: group-regret} is strictly more general.\\

\citet{Raveh2024} have an additive graph-dependent term that is strictly larger. Notably, their graph-dependence is always quadratic in the network size, whereas ours scales as $m d$. Notably, $d < m$ in connected graphs. 



\section{Additional Experiments}\label{appendix: additional-experiments}
Experimentally, we seek to compare our black-box reduction to existing works, who design and analyse extensions of well-known single agent bandit algorithms. Our simulations were conducted on a personal machine (Intel i5-10310U CPU @ 1.70Ghz x 8). Furthermore, all results are averages over $100$ independent runs. 

\subsection{Subgaussian Experiments}
Here, we describe the missing experimental details for the experiments found in the main paper, such as the hyperparameters for existing algorithms. Firstly, \citet{Landgren2016}, \citet{Rubio2019} and \citet{Lalitha2021} are all gossip-based algorithms. Thus, they all require the specification of a communication matrix. All authors recommend choosing this matrix based on the graph Laplacian. Formally, they use the following doubly stochastic matrix:
$$
    P = I - \frac{1}{1 + \max_{v \in V}\abs{N_{v}}}\left(D - M\right)
$$
where $M$, $D$ and $\delta$ denote the adjacency matrix, the degree matrix and the maximum degree of the underlying graph. Below, we provide specific details on additional hyperparameters of each algorithm.
\begin{itemize}
    \item \citet{Landgren2016} (coop-UCB) has been implemented by other researchers \citep{Rubio2019}. Their algorithm requires an exploration hyperparameter $\gamma > 1$; we choose this parameter based on the empirical results of existing work \citep{Rubio2019, Lalitha2021}.
    \begin{itemize}
        \item [\tiny$\bullet$] Exploration Parameter: $\gamma = 1.01$.
    \end{itemize}
    \item \citet{Rubio2019} (DDUCB) provide an implementation of their algorithm on \href{https://github.com/damaru2/decentralized-bandits}{Github}. Their algorithm requires specifying two hyperparameters; we use the values stated in their theorems and chosen in their experiments.
    \begin{itemize}
        \item [\tiny$\bullet$] Exploration Parameter (Theorem 3.2 of \citet{Rubio2019}): $\eta = 2$.
        \item [\tiny$\bullet$] Approximation Parameter (Theorem 3.2 of \citet{Rubio2019}): $\epsilon = \nicefrac{1}{22}$.
    \end{itemize}
    \item \citet{Lalitha2021} (DecTS) provide an implementation of their algorithm on \href{https://github.com/anushalalitha5/Decentralized-Thompson-Sampling}{Github}. Their algorithm requires the specification of a learning rate. We use the value found in the theoretical results that they use in their experiments. 
    \begin{itemize}
        \item [\tiny$\bullet$] Learning Rate (Theorem 1 of \citet{Lalitha2021}): $\eta = m$.
    \end{itemize}
    \item Algorithm \ref{algorithm: our-algorithm} requires the specification of a leader and a single-player bandit algorithm.
    \begin{itemize}
        \item [\tiny$\bullet$] Leader: As Theorem \ref{theorem: group-regret} suggests, we select the leader as the solution to the graph median problem for each experimental setting. We do so using Algorithm \ref{algorithm: best-leader}.
        \item [\tiny$\bullet$] Bandit Algorithm: Our goal is to compare how well our black-box reduction performs relative to existing works, who design and analyse specific algorithms designed for the setting. Therefore, we consider UCB \citep{Auer2002} and Thompson Sampling \citep{Thompson1933} to compare with existing literature. 
    \end{itemize}
\end{itemize}

Our experiments consider three network structures, namely cycle, grid and star graphs with varying network sizes. Figure \ref{figure: group-regret-196} in the main paper shows the experimental results for the largest network size of $m = 196$ agents. Here, we present the experimental results for smaller network sizes. 

\paragraph{Cycle Graph.}
The cycle graph consists of agents organised in a circle who are connected to their clockwise and anti-clockwise neighbours. Thus, the graph is $2$-regular. Section \ref{section: experimental-results} presents the empirical results for a cycle with $m = 196$ agents. Figure \ref{figure: cycle-regret} displays the group regret for cycle graphs with varying numbers of agents. 
\begin{figure*}[h!]
    \centering
    \includegraphics[]{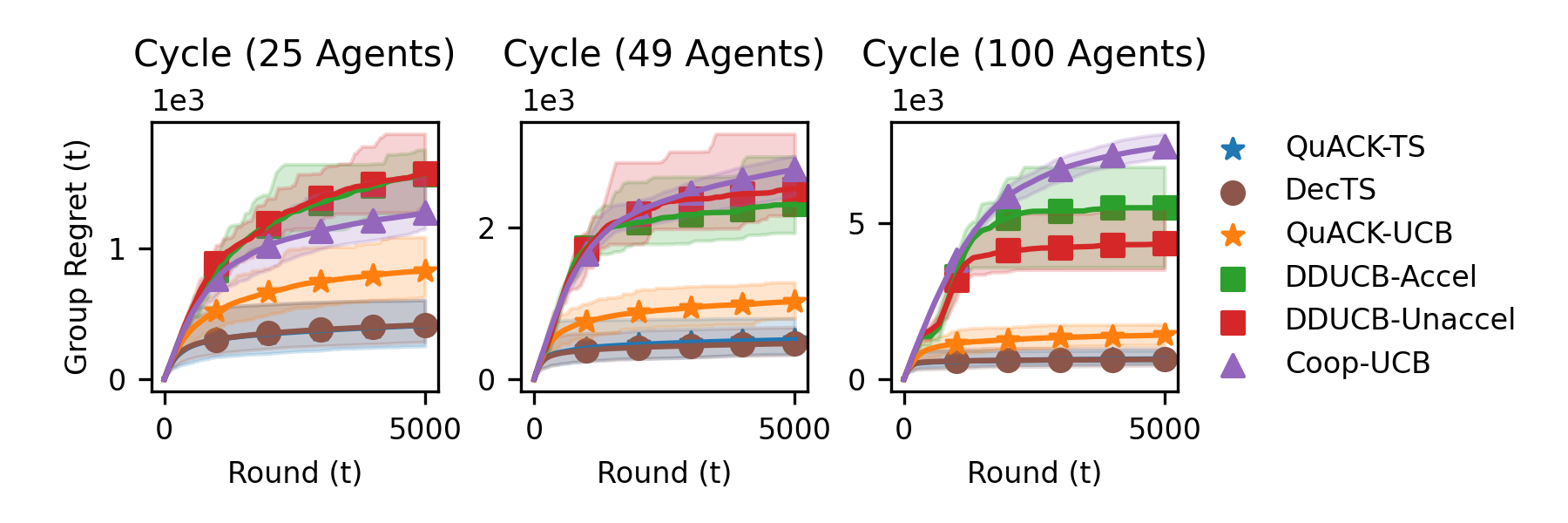}
    \caption{Group Regret for Cycle Graphs.}
    \label{figure: cycle-regret}
\end{figure*}

Figure \ref{figure: cycle-regret}, shows that QuACK-UCB outperforms the other UCB-based algorithms on all network sizes. As Section \ref{appendix: subgaussian} suggests, the performance gap increases with the size of the network. Furthermore, QuACK-TS is competitive with DecTS. 

\paragraph{Grid Graph.}
The grid graph consists of agents organised in a square lattice and each agent can have two, three or four neighbours. Section \ref{section: experimental-results} presents the results of our experiments for a grid graph with $m = 196$ agents. Figure \ref{figure: grid-regret} displays the group regret for grid graphs with varying numbers of agents. 
\begin{figure*}[h!]
    \centering
    \includegraphics[]{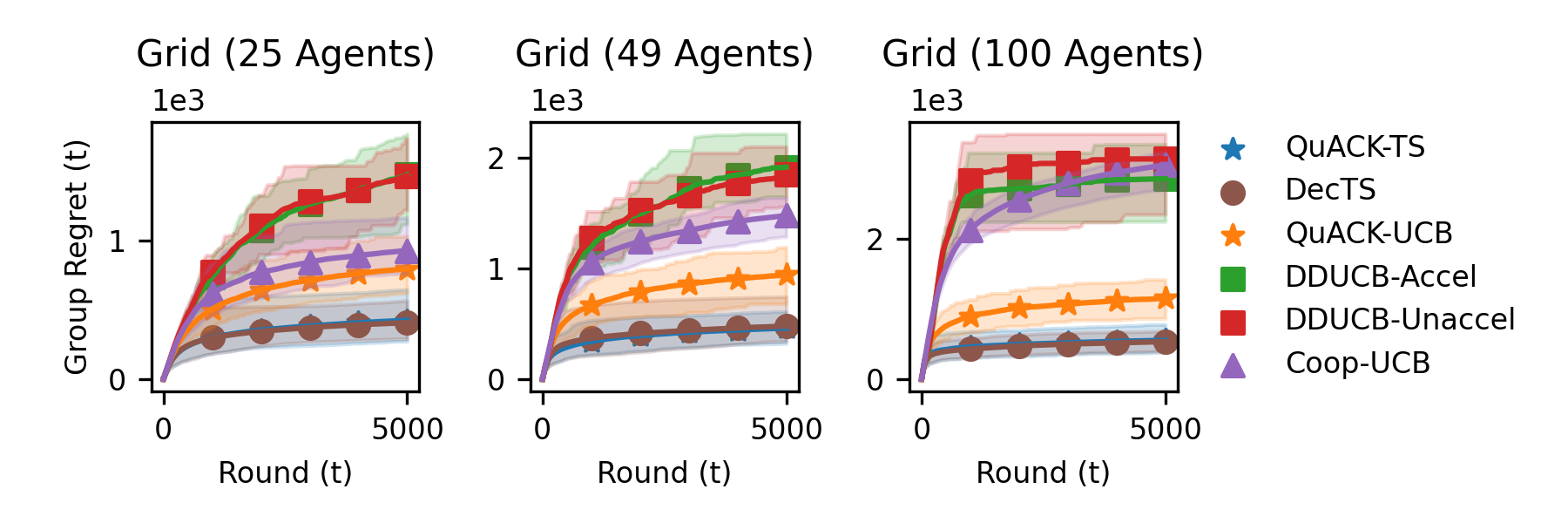}
    \caption{Group Regret for Grid Graphs.}
    \label{figure: grid-regret}
\end{figure*}

Here, QuACK-UCB out performs the other UCB-based algorithms on all network sizes. Furthermore, QuACK-TS is competitive with DecTS. 
\newpage
\paragraph{Star Graph.}
The star graph is a common sub-structure that appears in social networks. This particular graph consists of a single central vertex, who is connected to all other vertices. The non-central vertices have exactly one neighbour, which is the central vertex.  Section \ref{section: experimental-results} presents the empirical results for a star with $m = 196$ agents. However, we also conduct experiments on smaller stars to verify our theoretical results. Figure \ref{figure: grid-regret} displays the group regret for cycle graphs with varying numbers of agents. 
\begin{figure*}[h!]
    \centering
    \includegraphics[]{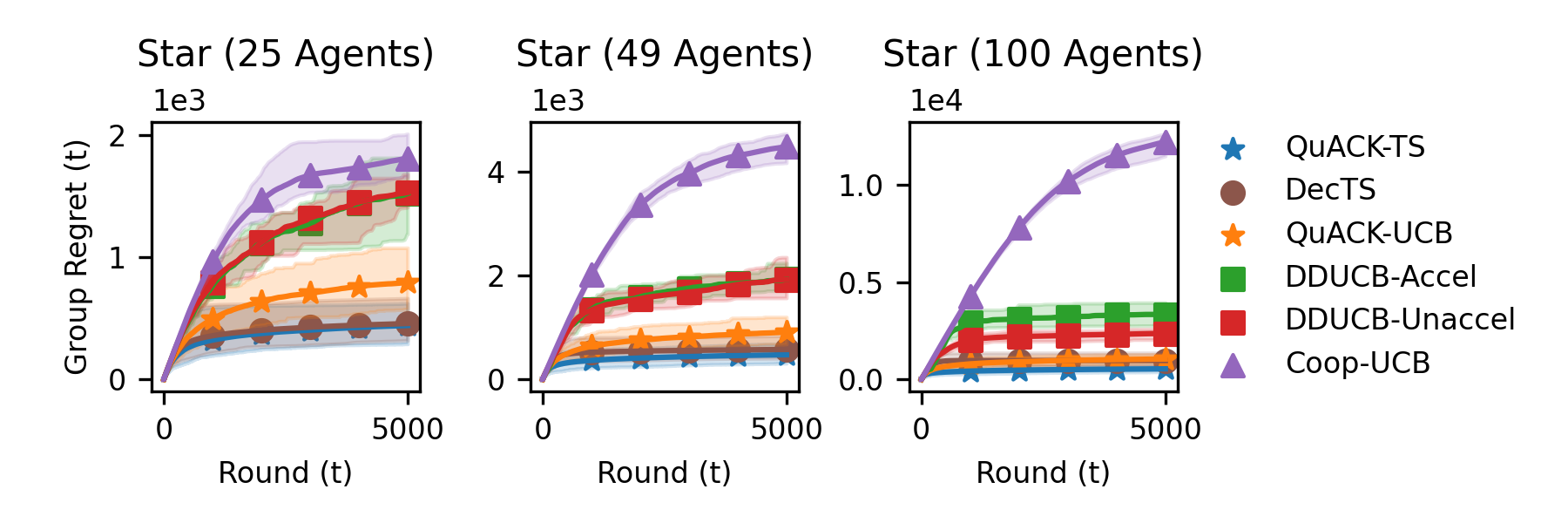}
    \caption{Group Regret for Star Graphs.}
    \label{figure: star-regret}
\end{figure*}

Here, QuACK-UCB out performs the other UCB-based algorithms on all network sizes. Furthermore, QuACK-TS is competitive with DecTS on the smallest star network. However, QuACK-TS achieves smaller group regret for larger star networks.

\paragraph{Network Scaling.}
Finally, Figure \ref{figure: scaling} visualises the group regret at the end of the final round as a function of the network size for each graph structure. 
\begin{figure*}[h!]
    \centering
    \includegraphics[]{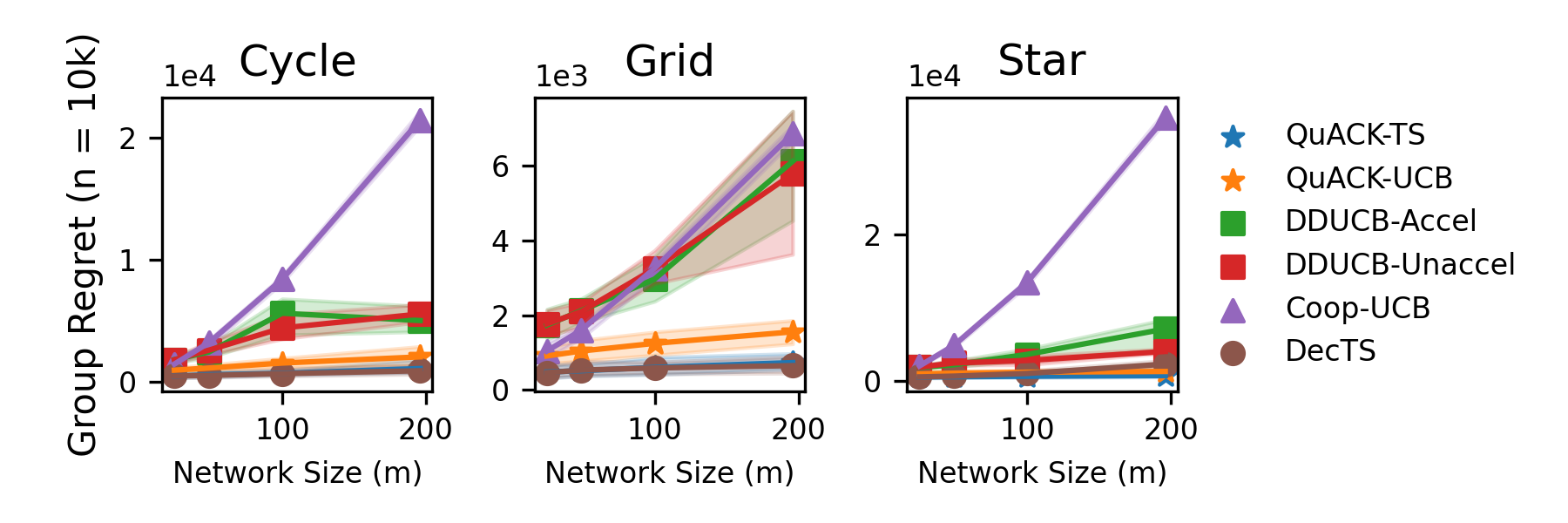}
    \caption{Network Size Scaling of Group Regret.}
    \label{figure: scaling}
\end{figure*}

Notably, we can see the group regret of Coop-UCB and DDUCB scales quickly with the network size for each graph. Conversely, the group regret of QuACK-UCB, QuACK-TS and Dec-TS have a shallower gradient, indicating a better dependence on the network size. 

\subsection{Heavy-Tailed Experiments}

Following \citet{Dubey2020}, we conduct our experiments on standard $\alpha$-stable densities. Specifically, $\alpha = 1.9$ in our experiments and we select the location of the densities such that:
$$
    \mu_{a} = 
    \begin{cases}
        0.7 & \text{ if } a = 1\\
        0.4 & \text{ if } a \neq 1
    \end{cases}
$$
where $\mathcal{A} = \{1, 2, \cdots, 5\}$. Our experiments in this setting consider the cycle and grid network structures with $m \in \{16, 36, 64\}$ agents. These experiments use a smaller number of agents and fewer rounds because the computational complexity of computing the truncated mean estimator grows linearly with the number of reward samples for each action.\footnote{The number of rewards for each action depends implicitly on the network size in the multi-agent setting.}\\

Below, we provide specific details of the hyperparameters used for each algorithms. 
\begin{itemize}
    \item \citet{Dubey2020} obtain the best theoretical guarantees for CMP-UCB and this algorithm partitions the graph into disjoint subsets. Each of these subset has one leader that uses the robust upper confidence bound algorithm. Following their guidance, we select the leaders such that they form the maximal weighted independent set of each graph.
    \begin{itemize}
        \item [\tiny$\bullet$] We tune the hyperparameter $\gamma$ of their algorithm such that the networks under consideration have $1$, $2$ and $4$ leaders. With $1$ leader we expect their algorithm to outperform ours. With $\geq 2$ leaders, we expect our algorithm to perform best.
    \end{itemize}
    \item Algorithm \ref{algorithm: our-algorithm} requires a single-agent heavy-tailed bandit algorithm as input. For a direct comparison, we choose the robust upper confidence bound algorithm with the truncated mean. We select the leader as the $\argmin$ of the graph-median problem.
\end{itemize}

\begin{figure*}[h!]
    \centering
    \includegraphics[]{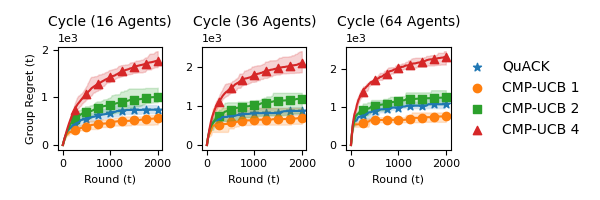}
    \includegraphics[]{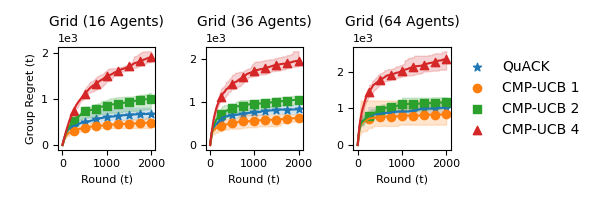}
    \caption{Group Regret for Cycle and Grid Graphs.}
    \label{figure: heavy-experiments}
\end{figure*}

Figure \ref{figure: heavy-experiments} displays the group regret for the various cycle and grid graphs. From Section \ref{appendix: heavy-compare}, we expect QuACK will perform better than CMP-UCB when it receives a $\gamma$ such that the resulting graph has independence number is greater than $1$. Notably, our experimental results align exactly with our theoretical predictions.
\end{appendix}
\end{document}